\documentclass[twoside,11pt]{article}
\usepackage{arxiv} 
\usepackage{hyperref}
\usepackage{changes}
\usepackage{url}
\usepackage{amssymb}
\usepackage{enumitem}
\usepackage{graphicx}
\usepackage{color}
\usepackage{subcaption}
\usepackage{hyperref}
\usepackage{amsmath}
\usepackage{amssymb}
\usepackage{amsthm}
\usepackage{graphicx}
\usepackage{tabularx}
\usepackage{fancyvrb}
\usepackage{comment}
\usepackage{algorithm}
\usepackage{algpseudocode}
\usepackage{float}
\usepackage{wrapfig}
\usepackage{listings}
\usepackage{natbib}
\usepackage{mathtools}
\usepackage{booktabs}
\usepackage{siunitx}

\newcommand{\xspace}{\mathcal{X}}
\newcommand{\yspace}{\mathcal{Y}}
\newcommand{\Exp}{\mathbb{E}}

\newcommand{\reals}{{\mathbb{R}}}

\newtheorem{lem}{Lemma}[section]

\newtheorem{asmp}{Assumption}[section]
\newtheorem{defn}{Definition}[section]

\newtheorem{rem}{Remark}[section]

\def\approxcorrect{\checkmark\kern-1.1ex\raisebox{.89ex}{$\times$}}

\usepackage{amsmath,amsfonts,bm}

\def\eqref#1{equation~\ref{#1}}

\def\1{\bm{1}}

\usepackage{thmtools,thm-restate}

\theoremstyle{definition}

\floatname{algorithm}{Procedure}
\renewcommand{\algorithmicrequire}{\textbf{Input:}}
\renewcommand{\algorithmicensure}{\textbf{Output:}}

\DeclarePairedDelimiterX{\inp}[2]{\langle}{\rangle}{#1, #2}

\definecolor{dkgreen}{rgb}{0,0.6,0}
\definecolor{gray}{rgb}{0.5,0.5,0.5}
\definecolor{mauve}{rgb}{0.58,0,0.82}

\ShortHeadings{Algorithms and Theory for Supervised Gradual Domain Adaptation}{Jing Dong, Shiji Zhou, Baoxiang Wang and Han Zhao}
\firstpageno{1}
\begin{document}
\allowdisplaybreaks
\title{Algorithms and Theory for Supervised Gradual Domain Adaptation}
\author{\name Jing Dong \email jingdong@link.cuhk.edu.cn  \\
       \addr The Chinese University of Hong Kong, Shenzhen\\
       \name Shiji Zhou \email zsj17@mail.tsinghua.edu.cn \\ \addr Tsinghua-Berkeley Shenzhen Institute, Tsinghua University\\
       \name Baoxiang Wang \email bxiangwang@cuhk.edu.cn  \\
       \addr The Chinese University of Hong Kong, Shenzhen\\
       \name Han Zhao \email hanzhao@illinois.edu \\
       \addr University of Illinois Urbana-Champaign}
\maketitle
\begin{abstract}
The phenomenon of data distribution evolving over time has been observed in a range of applications, calling for the need for adaptive learning algorithms. We thus study the problem of supervised gradual domain adaptation, where labeled data from shifting distributions are available to the learner along the trajectory, and we aim to learn a classifier on a target data distribution of interest. Under this setting, we provide the first generalization upper bound on the learning error under mild assumptions. Our results are algorithm agnostic, general for a range of loss functions, and only depend linearly on the averaged learning error across the trajectory. This shows significant improvement compared to the previous upper bound for unsupervised gradual domain adaptation, where the learning error on the target domain depends exponentially on the initial error on the source domain. Compared with the offline setting of learning from multiple domains, our results also suggest the potential benefits of the temporal structure among different domains in adapting to the target one. Empirically, our theoretical results imply that learning proper representations across the domains will effectively mitigate learning errors. Motivated by these theoretical insights, we propose a min-max learning objective to learn the representation and classifier simultaneously. Experimental results on both semi-synthetic and large-scale real datasets corroborate our findings and demonstrate the effectiveness of our objectives. 
\end{abstract}

\section{Introduction}
An essential assumption for the deployment of machine learning models in real-world applications is the alignment of training and testing data distributions. Under this condition, models are expected to generalize, yet real-world applications often fail to meet this assumption. Instead, continual distribution shift is widely observed in a range of applications. For example, satellite images of buildings and lands change over time due to city development \citep{christie2018functional}; self-driving cars receive data with quality degrading towards nightfall \citep{bobu2018adapting,wu2019ace}. Although this problem can be mitigated by collecting training data that covers a wide range of distributions, it is often impossible to obtain such a large volume of labeled data in many scenarios. On the other hand, the negligence of shifts between domains also leads to suboptimal performance. Motivated by this commonly observed phenomenon of gradually shifting distributions, we study supervised gradual domain adaptation in this work. Supervised gradual domain adaptation models the training data as a sequence of batched data with underlying changing distributions, where the ultimate goal of learning is to obtain an effective classifier on the target domain at the last step. This relaxation of data alignment assumption thus equips gradual domain adaptation with applicability in a wide range of scenarios. Compared with unsupervised gradual domain adaptation, where only unlabeled data is available along the sequence, in supervised gradual domain adaptation, the learner also has access to labeled data from the intermediate domains. Note that this distinction in terms of problem setting is essential, as it allows for more flexible model adaptation and algorithm designs in supervised gradual domain adaptation. 

The mismatch between training and testing data distributions has long been observed, and it had been addressed with conventional domain adaptation and multiple source domain adaptation \citep{duan2012learning,hoffman2013efficient,hoffman2018cycada,hoffman2018algorithms,zhao2018adversarial,wen2020domain,mansour2021theory} in the literature. Compared with the existing paradigms, supervised gradual domain adaptation poses new challenges for these methods, as it involves more than one training domains and the training domains come in sequence. For example, in the existing setting of multiple-source domain adaptation~\citep{zhao2018adversarial,hoffman2018algorithms}, the learning algorithms try to adapt to the target domain in a one-off fashion. Supervised gradual domain adaptation, however, is more realistic, and allows the learner to take advantage of the temporal structure among the gradually changing training domains, which can lead to potentially better generalization due to the smaller distributional shift between each consecutive pair of domains. 

Various empirically successful algorithms have been proposed for gradual domain adaptation \citep{hoffman14,gadermayr2018gradual,wulfmeier2018incremental,bobu2018adapting}. Nevertheless, we still lack a theoretical understanding of their limits and strengths. The first algorithm-specific theoretical guarantee for unsupervised gradual domain adaptation is provided by \citet{kumar2020understanding}. However, the given upper bound of the learning error on the target domain suffers from exponential dependency (in terms of the length of the trajectory) on the initial learning error on the source domain. This is often hard to take in reality and it is left open whether this can be alleviated in supervised gradual domain adaptation.  

In this paper, we study the problem of gradual domain adaptation under a supervised setting where labels of training domains are available. We prove that the learning error of the target domain is only linearly dependent on the averaged error over training domains, showing a significant improvement compared to the unsupervised case. We show that our results are comparable with the learning bound for multiple source training and can be better under certain cases while relaxing the requirement of access to all training domains upfront simultaneously. Further, our analysis is algorithm and loss function independent. Compared to previous theoretical results on domain adaptation, which used $l_1$ distance \citep{mansour2009domain} or $W_\infty$ distance to capture shifts between data distributions \citep{kumar2020understanding}, our results are obtained under milder assumptions. We use $W_p$ Wasserstein distance to describe the gradual shifts between domains, enabling our results to hold under a wider range of real applications. Our bound features two important ingredients to depict the problem structure: sequential Rademacher complexity \citep{rakhlin2015online} is used to characterize the sequential structure of gradual domain adaptation while discrepancy measure \citep{kuznetsov2017generalization} is used to measure the non-stationarity of the sequence. 

Our theoretical results provide insights into empirical methods for gradual domain adaptation. Specifically, our bound highlights the following two observations: (1) Effective representation where the data drift is ``small'' helps. Our theoretical results highlight an explicit term showing that representation learning can directly optimize the learning bound. (2) There exists an optimal time horizon (number of training domains) for supervised gradual domain adaptation. Our results highlight a trade-off between the time horizon and the learning bound. 

Based on the first observation, we propose a min-max learning objective to learn representations concurrently with the classifier. Optimizing this objective, however, requires simultaneous access to all training domains. In light of this challenge, we relax the requirement of simultaneous access with temporal models that encode knowledge of past training domains. To verify our observations and the proposed objectives, we conduct experiments on both semi-synthetic datasets with MNIST dataset and large-scale real datasets such as FMOW \citep{christie2018functional}. Comprehensive experimental results validate our theoretical findings and confirm the effectiveness of our proposed objective.

\section{Related Work}
\paragraph{(Multiple source) domain adaptation}
Learning with shifting distributions appears in many learning problems. Formally referred as domain adaptation, this has been extensively studied in a variety of scenarios, including computer vision \citep{hoffman14,venkateswara2017deep,zhao2019madan}, natural language processing \citep{blitzer2006domain,blitzer2007biographies,axelrod2011domain}, and speech recognition \citep{sun2017unsupervised,sim2018domain}. When the data labels of the target domain are available during training, known as supervised domain adaptation, several parameter regularization-based methods \citep{yang2007adapting,aytar2011tabula}, feature transformations based methods \citep{saenko2010adapting,kulis2011you} and a combination of the two are proposed \citep{duan2012learning,hoffman2013efficient}. 
On the theoretical side of domain adaptation, \citet{david2010impossibility} investigated the necessary assumptions for domain adaptations, which specifies that either the domains are needed to be similar, or there exists a classifier in the hypothesis class that can attain low error on both domains. For the restriction on the similarity of domains, \citet{wu2019domain} proposed an asymmetrically-relaxed distribution alignment for an alternative requirement. This is a more general condition for domain adaptation when compared to those proposed by \citet{david2010impossibility}, and holds in a more general setting with high-capacity hypothesis classes such as neural networks. \citet{zhao2019learning} then focused on the efficiency of representation learning for domain adaption and characterized a trade-off between learning an invariant representation across domains and achieving small errors on all domains. This is then extended to a more general setting by \citet{zhao2020fundamental}, which provides a geometric characterization of feasible regions in the information plane for invariant representation learning. The problem of adapting with multiple training domains, referred to as multiple source domain adaptation (MDA), is also studied extensively. The first asymptotic learning bound for MDA is studied by \citet{hoffman2018algorithms}. Follow-up work \citet{zhao2018adversarial} provides the first generalization bounds and proposed efficient adversarial neural networks to demonstrate empirical superiority. The theoretical results are further explored by \citet{wen2020domain} with a generalized notion of distance measure, and by \citet{mansour2021theory} when only limited target labeled data are available. 

\paragraph{Gradual domain adaptation}~~
Many real-world applications involve data that come in sequence and are continuously shifting. This first attempt addresses the setting where data from continuously evolving distribution with a novel unsupervised manifold-based adaptation method \citep{hoffman14}. Following works \citet{gadermayr2018gradual,wulfmeier2018incremental,bobu2018adapting} also proposed unsupervised approaches for this variant of gradual domain adaptation with unsupervised algorithms. The first to study the problem of adapting to an unseen target domain with shifting training domains is \citet{kumar2020understanding}. Their result features the first theoretical guarantee for unsupervised gradual domain adaptation with a self-training algorithm and highlights that learning with a gradually shifting domain can be potentially much more beneficial than a Direct Adaptation. The work provides a theoretical understanding of the effectiveness of empirical tricks such as regularization and label sharpening. However, they are obtained under rather stringent assumptions. They assumed that the label distribution remains unchanged while the varying class conditional probability between any two consecutive domains has bounded $W_\infty$ Wasserstein distance, which only covers a limited number of cases. Moreover, the loss functions are restricted to be the hinge loss and ramp loss while the classifier is restricted to be linear. Under these assumptions, the final learning error is bounded by $O(\exp(T)\alpha_0)$, where $T$ is the horizon length and $\alpha_0$ is the initial learning error on the initial domain. This result is later extended by \citet{chen2020self} with linear classifiers and Gaussian spurious features and improved by concurrent independent work \citet{wang2022understanding} to $O(\alpha_0 + T)$ in the setting of unsupervised gradual domain adaptation. The theoretical advances are complemented by recent empirical success in gradual domain adaptation. Recent work \citet{chen2021gradual} extends the unsupervised gradual domain adaptation problem to the case where intermediate domains are not already available. \citet{abnar2021gradual} and \citet{sagawa2021extending} provide the first comprehensive benchmark and datasets for both supervised and unsupervised gradual domain adaptation. 

\section{Preliminaries}
The problem of gradual domain adaptation proceeds sequentially through a finite time horizon $\{1, \dots, T\}$ with evolving data domains. A data distribution $P_t \in \reals^d \times \reals^k $ is realized at each time step with the features denoted as $X \in \reals^d$ and labels as $Y \in \reals^k$. With a given loss function $\ell( \cdot, \cdot) $, we are interested in obtaining an effective classifier $h \in \mathcal{H}: \reals^d \rightarrow \reals^k$ that minimizes a given loss function on the target domain $P_T$, which is also the last domain. With access to only $n$ samples from each intermediate domain $P_1, \ldots, P_{T-1}$, we seek to design algorithms that output a classifier at each time step where the final classifier performs well on the target domain. 

Following the prior work~\citep{kumar2020understanding}, we assume the shift is gradual. To capture such a gradual shift, we use the Wasserstein distance to measure the change between any two consecutive domains. The Wasserstein distance offers a way to include a large range of cases, including the case where the two measures of the data domains are not on the same probability space~\citep{cai2020distances}. 
\begin{defn}{(Wasserstein distance)}
The $p$-th Wasserstein distance, denoted as $W_p$ distance, between two probability distribution $P, Q$ is defined as 
$$
    W_p(P, Q) = \left( \inf_{\gamma \in \Gamma(P,Q)} \int \| x - y\|^p d \gamma (x,y)\right)^{1/p} \,,
$$
where $\Gamma(P,Q)$ denotes the set of all joint distribution $\gamma$ over $(X, Y)$ such that $X \sim P$, $Y \sim Q$.
\end{defn}
Intuitively, Wasserstein distance measures the minimum cost needed to move one distribution to another. The flexibility of Wasserstein distance enables us to derive tight theoretical results for a wider range of practical applications. In comparison, previous results leverage $l_1$ distance \citep{mansour2009domain} or the Wasserstein-infinity $W_\infty$ distance \citep{kumar2020understanding} to capture non-stationarity. In practice, however, this is rarely used and $W_1$ is more commonly employed due to its low computational cost. Moreover, even if pairs of data distributions are close to each other in terms of $W_1$ distance, the $W_\infty$ distance can be unbounded with a few presences of outlier data. Previous literature hence offers limited insights whereas our results include this more general scenario. 
We formally describe the assumptions below. 
\begin{asmp}\label{asmp:gradual}
For all $1 \leq t \leq T-1$ and some constant $\Delta > 0$, the $p$-th Wasserstein distance between two domains is bounded as $W_p(P_{t} (X,Y), P_{t+1}(X,Y)) \leq \Delta$. 
\end{asmp}
We study the problem without restrictions on the specific form of the loss function, and we only assume that the empirical loss function is bounded and Lipschitz continuous. This covers a rich class of loss functions, including the logistic loss/binary cross-entropy, and hinge loss. Formally, let $\ell_h$ be the loss function, $\ell_h = \ell(h(x), y):\xspace\times\yspace\to \mathbb{R}$. We have the following assumption. 
\begin{asmp}\label{asmp:lip}
The loss function $\ell_h : \mathcal{X} \times \mathcal{Y} \rightarrow \mathbb{R}$ is $\rho$-Lipschitz continuous with respect to $(x,y)$, denoted as $\|\ell_h\|_{\text{Lip}}\leq \rho$, and bounded such that $\|\ell_h \|_\infty \leq M $.
\end{asmp}
This assumption is general as it holds when the input data are compact. 
Moreover, we note that this assumption is mainly for the convenience of technical analysis and is common in the literature \citep{mansour2009domain,cortes2011domain,kumar2020understanding}.

The gradual domain adaptation investigated in this paper is thus defined as follows
\begin{defn}[Gradual Domain Adaptation]
Given a finite time horizon $\{1, \dots, T\}$, at each time step $t$, a data distribution $P_t$ over $\reals^d \times \reals^k$ is realized with the features $X \in \reals^d$ and labels $Y \in \reals^k$. The data distributions are gradually evolving subject to Assumption \ref{asmp:gradual}. With a given loss function $\ell( \cdot, \cdot) $ that satisfies Assumption \ref{asmp:lip}, the goal is to obtain a classifier $h \in \mathcal{H}: \reals^d \rightarrow \reals^k$ that minimizes the given loss function on the target domain $P_T$ with access to only $n$ samples from each intermediate domain $P_1, \ldots, P_{T-1}$.
\end{defn}



Our first tool is used to help us characterize the structure of sequential domain adaptation. Under the statistical learning scenario with i.i.d.\ data, Rademacher complexity serves as a well-known complexity notion to capture the richness of the underlying hypothesis space. However, with the sequential dependence, classical notions of complexity are insufficient to provide a description of the problem. To capture the difficulty of sequential domain adaptation, we use the sequential Rademacher complexity, which was originally proposed for online learning where data comes one by one in sequence \citep{rakhlin2015online}. 



\begin{defn}[$\mathcal{Z}$-valued tree \citet{rakhlin2015online}] \label{def:z_valued}
A $\mathcal{Z}$-value tree $z$ is a sequence $(z_1, \ldots, z_T)$ of T mappings, $z_t: \{ \pm 1\}^{t-1} \rightarrow \mathcal{Z}, t \in [1, T]$. A path in the tree is $\epsilon = (\epsilon_1, \ldots, \epsilon_{T-1}) \in \{ \pm 1\}^{T-1}$. To simplify the notations, we write $z_t ( \epsilon) = z_t ( \epsilon_1, \ldots, \epsilon_{t-1})$.
\end{defn}

\begin{defn}[Sequential Rademacher Complexity \citet{rakhlin2015online}]\label{def:seq_complexity}
For a function class $\mathcal{F}$, the sequential Rademacher complexity is defined as 
$
    \mathfrak{R}_{T}^{\mathrm{seq}}(\mathcal{F})=\sup_{\mathbf{z}} \mathbb{E}\left[\sup _{f \in \mathcal{F}} \frac{1}{T}\sum_{t=1}^{T} \epsilon_{t}  f\left(z_{t}(\epsilon)\right)\right] \,,
$
where the supremum is taken over all $\mathcal{Z}$-valued trees (Definition \ref{def:z_valued}) of depth $T$ and $(\epsilon_1, \ldots, \epsilon_{T})$ are Rademacher random variables.
\end{defn}

We next introduce the discrepancy measure, a key ingredient that helps us to characterize the non-stationarity resulting from the shifting data domains. This can be used to bridge the shift in data distribution with the shift in errors incurred by the classifier. To simplify the notation, we let $Z = (X,Y)$ and use shorthand $Z_1^T$ for $Z_1, \ldots, Z_T$. 
\begin{defn}[Discrepancy measure \citet{kuznetsov2020discrepancy}]
\label{def:discrepancy}
\begin{align}
    \operatorname{disc}_T = &\sup _{h \in \mathcal{H}}\left(\mathbb{E}\left[\ell_h\left(X_{T}, Y_T\right) \mid Z_{1}^{T-1}\right] -  \frac{1}{T}\sum_{t=1}^{T}  \mathbb{E}\left[ \ell_h\left(X_{t}, Y_t\right) \mid Z_{1}^{t-1}\right]\right) \,,
\end{align}
where $Z_1^0$ is defined to be the empty set, in which case the expectation is equivalent to the unconditional case.
\end{defn}

We will later show that the discrepancy measure can be directly upper-bounded when the shift in class conditional distribution is gradual. We also note that this notion is general and feasible to be estimated from data in practice \citep{kuznetsov2020discrepancy}. Similar notions have also been used extensively in non-stationary time series analysis and mixing processes \citep{kuznetsov2014generalization,kuznetsov2017generalization}. 

\section{Theoretical Results}
In this section, we provide our theoretical guarantees for the performance of the final classifier learned in the setting described above. Our result is algorithm agnostic and general to loss functions that satisfy Assumption \ref{asmp:lip}. We then discuss the implications of our results and give a proof sketch to illustrate the main ideas. 

The following theorem gives an upper bound of the expected loss of the learned classifier on the last domain in terms of the shift $\Delta$, sequential Rademacher complexity, etc. 

\begin{restatable}{thm}{mainthm} \label{thm:main}
Under Assumptions \ref{asmp:gradual}, \ref{asmp:lip}, with $n$ data points access to each data distribution $P_t$, $t \in \{1, \ldots, T\}$, and loss function $\ell_h = \ell(h(x), y):\xspace\times\yspace\to \mathbb{R}$, the loss on the last distribution incurred by a classifier $h_T$ obtained through empirically minimizing the loss on each domains can be upper bounded by \begin{align}\label{eq:thm}
&\mathbb{E}\left[\ell_{h_T}\left(X_T, Y_{T}\right) \mid Z_{1}^{T-1}\right] \nonumber \\ 
&\leq \mathbb{E}\left[\ell_{h_0}\left(X_T, Y_T\right) \mid Z_{1}^{T-1}\right] + \underbrace{\frac{3}{T} +\frac{3M}{T} \sqrt{8 \log \frac{1}{\delta}}}_{E_1} + \underbrace{\frac{1}{T}\sqrt{\frac{\text{VCdim}(\mathcal{H}) + \log (2/\delta)}{2n}} + O\left(\frac{1}{\sqrt{nT}}\right)}_{E_2} \nonumber \\
&+ \underbrace{18 M \sqrt{4 \pi \log T} \mathfrak{R}_{T-1}^{s e q}(\mathcal{F})+ 3T\rho\Delta}_{E_3}\,,
\end{align}
where $\ell_h \in \mathcal{F}$, $\mathfrak{R}_{T}^{s e q}(\mathcal{F})$ is the sequential Rademacher complexity of $\mathcal{F}$, $\text{VCdim}(\mathcal{H})$ is the VC dimension of $\mathcal{H}$ and $h_{0}=\operatorname{argmin}_{h \in \mathcal{H}} \frac{1}{T}\sum_{t=1}^{T} \ell\left(h(X_t), Y_{t}\right)$.
\end{restatable}

When $\ell_h \in \mathcal{F}$ is bounded and convex, the sequential Rademacher complexity term is upper bounded by $O(\sqrt{1/nT})$ \citep{rakhlin2015online}. For some complicated function classes, such as multi-layer neural networks, they also enjoy a sequential Rademacher complexity of order $O(\sqrt{1/nT})$ \citep{rakhlin2015online}. Before we move to present a proof sketch of Theorem \ref{thm:main}, we first discuss the implications of our theorem.
\begin{rem}\label{rem:1}
There exists a non-trivial trade-off between $E_1 + E_2$ and $E_3$ through the length $T$. When $T$ is larger, all terms except for the terms in $E_3$ will be smaller while the terms in $E_3$ will be larger. Hence, it is not always beneficial to have a longer trajectory.
\end{rem}
\begin{rem}\label{rem:2}
All terms in (\ref{eq:thm}) except for the last term $3T\rho\Delta$ are determined regardless of the algorithm. 
The last term depends on $\Delta$ which measures the class conditional distance between any two consecutive domains. This distance can potentially be minimized through learning an effective representation of data.
\end{rem}
\paragraph{Comparison with unsupervised gradual domain adaptation}
\label{rem:linear}
Our result is only linear with respect to the horizon length $T$ and the 
average loss $\mathbb{E}\left[\ell_{h_0}\left(X_T, Y_T\right) \mid Z_{1}^{T-1}\right]$, where $h_{0}=\operatorname{argmin}_{h \in \mathcal{H}} $ $\frac{1}{T}\sum_{t=1}^{T} \ell\left(h(X_t), Y_{t}\right)$. In contrast, the previous upper bound given by \citet{kumar2020understanding}, which is for unsupervised gradual domain adaptation, is $O(\exp(T)\alpha_0)$, with $\alpha_0$ being the initial error on the initial domain. It remains unclear, however, if the exponential cost is unavoidable when labels are missing during training as the result by \citet{kumar2020understanding} is algorithm specific. 

\paragraph{Comparison with multiple source domain adaptation}
The setting of multiple source domain adaptation neglects the temporal structure between training domains. Our results are comparable while dropping the requirement of simultaneous access to all training domains. Our result suffers from the same order of error with respect to the Rademacher complexity and from the VC inequality with supervised multiple source domain adaptation (MDA) \citep{wen2020domain}, which is $O\left( \sum_{t}\alpha_t( \ell_{h_T}(X_t, Y_t) + \mathfrak{R}(\mathcal{F}) \right)$, where $\sum_t \alpha_t = 1, \alpha_t > 0, \forall t$ and $\mathfrak{R}(\mathcal{F})$ being the Rademacher complexity.Taking the weights $\alpha_t = \frac{1}{T}$, the error of a classifier $h$ on the target domain similarly relies on the average error of $h$ on training domains. We note that in comparison our results scale with the averaged error of the best classifier on the training domains. 

While we defer the full proof to the appendix, we now present a sketch of the proof. 

\textit{Proof Sketch }
With Assumption \ref{asmp:gradual}, we first show that when the Wasserstein distance between two consecutive class conditional distributions is bounded, the discrepancy measure is also bounded. 
\begin{restatable}{lem}{losswpbound}\label{lem:loss_wp_bound}
Under Assumption \ref{asmp:lip}, the expected loss on two consecutive domains satisfy
$
    \Exp_\mu[\ell_h(X, Y)] - \Exp_\nu[\ell_h(X^\prime, Y^\prime)] \leq \rho\Delta \,,
$
where $\mu, \nu$ are the probability measure for $P_t, P_{t+1}$, $(X, Y) \sim P_t$, and $(X^\prime, Y^\prime) \sim P_{t+1}$. 
\end{restatable}

Then we leverage this result to bound the loss incurred in expectation by the same classifier on two consecutive data distributions. We start by decomposing the discrepancy measure with an adjustable summation term as
\begin{align*}
\operatorname{disc}_T 
\leq & \sup_{h \in \mathcal{H}}\left(\frac{1}{s} \sum_{t=T-s+1}^{T} \mathbb{E}\left[\ell_h\left(X_{t}, Y_t\right) \mid Z_{1}^{t-1}\right] - \frac{1}{T} \sum_{t=1}^{T}  \mathbb{E}\left[ \ell_h\left(X_{t}, Y_t\right) \mid Z_{1}^{t-1}\right]\right) \\
&+ \sup _{h \in \mathcal{H}}\left(\mathbb{E}\left[\ell_h\left(X_{T}, Y_{T}\right) \mid Z_{1}^{T-1}\right] -  \frac{1}{s} \sum_{t=T-s+1}^{T} \mathbb{E}\left[ \ell_h \left(X_{t}, Y_t\right) \mid Z_{1}^{t-1}\right]\right) \,.
\end{align*}
We show by manipulating this adjustable summation, the discrepancy measure can indeed be directly obtained through an application of Lemma \ref{lem:loss_wp_bound}. We now start to bound the learning error in interest by decomposing 
\begin{align*}
    & \mathbb{E}\left[\ell_{h_{T}}\left(X_{T}, Y_{T}\right)\mid Z_{1}^{T-1}\right]- \mathbb{E}\left[\ell_{h_0}\left(X_T, Y_T\right) \mid Z_{1}^{T-1}\right] \\
    \leq & 2\Phi(Z_1^T) + \left( \frac{1}{T}\sum^{T-1}_{t=1} \left[\ell_{h_{T}}\left(X_{t}, Y_{t}\right)\right] - \frac{1}{T}\sum^{T-1}_{t=1}\ell_{h_0}\left(X_T, Y_T\right)  \right)\,,
\end{align*}
where $\Phi\left(Z_{1}^{T}\right)=\sup _{h \in \mathcal{H}}\left(\mathbb{E}\left[\ell_h\left(X_{T}, Y_{T}\right) \mid Z_{1}^{T-1}\right] \right.$ 
$\left.-\sum_{t=1}^{T} \frac{1}{T} \ell_h\left(X_{t}, Y_{t}\right)\right)$.
The term $\Phi\left(Z_{1}^{T}\right)$ can be upper bounded by Lemma \ref{lem:cor2} \citet{kuznetsov2020discrepancy} and thus it is left to bound the remaining term $\frac{1}{T}\sum^{T-1}_{t=1} \left[\ell_{h_{T}}\left(X_{t}, Y_{t}\right)\right] - \frac{1}{T}\sum^{T-1}_{t=1}\ell_{h_0}\left(X_T, Y_T\right) $. To upper bound this difference of average loss, we first compare the loss incurred by a classifier learned by an optimal online learning algorithm to $f_0$. By classic online learning theory results, the difference is upper bounded by $O\left(\frac{1}{\sqrt{nT}}\right)$. Then we compare the optimal online learning classifier to our final classifier $h_T$ and upper bound the difference through the VC inequality \citep{bousquet2004introduction}. 

Lastly, we leverage Corollary 3 of \citet{kuznetsov2020discrepancy} with our terms to complete the proof. 
\hfill $\square$

\section{Insights for Practice}
The key insight indicated by Theorem \ref{thm:main} and Remark \ref{rem:2} is that the bottleneck of supervised gradual domain adaption is not only predetermined through the setup of the problem but also relies heavily on $\rho \Delta$, where $\Delta $ is the upper bound of the Wasserstein class conditional distance between two data domains and $\rho$ is the Lipschitz constant of the loss function. In practice, the loss function is often chosen beforehand and remains unchanged throughout the learning process. Therefore, the only term available to be optimized is $\Delta$, which can be effectively reduced if a good representation of data can be learned for classification. We give a feasible primal-dual objective that learns a mapping function from input to feature space concurrently with the original classification objective

\paragraph{A primal-dual objective formulation}
Define $g$ to be a mapping that maps $X \in \mathbb{R}^d$ to some feature space. We propose the learning objective as to learn a classifier $h$ simultaneously with the mapping function $g$ with the exposure of historical data $Z_1^{T-1}$. With the feature $g(X)$ from the target domain, our learning objective is now
$
\mathbb{E}\left[\ell_h(g(X_{T}), Y_T))|Z_1^{T-1} \right] 
-  \inf_{h^\ast,g^\ast} \mathbb{E}\left[\ell_{h^\ast} (g^\ast(X_{T}), Y_T)|Z_1^{T-1} \right] 
$.
Intuitively, this can be viewed as a combination of two optimization problems where both $\Delta$ and the learning loss are minimized.

The objective is hard to evaluate without further assumptions. Thus we restrict our study to the case where both $g$ and $h$ are parametrizable. Specifically, we assume $g$ is parameterized by $\omega$ and $h$ is parameterized by $\theta$. Then we leverage the Wasserstein-$1$ distance's dual representation (Equation \ref{eq:dual}, \citet{KR58}) to derive an objective for learning the representation, 
\begin{align}\label{eq:dual}
    W_1 (P, Q) = \sup_{\gamma \in \Gamma(P, Q)} \int \gamma(x) d P(x) - \int \gamma(y) dQ(y) = \sup_{\gamma \in \Gamma(P, Q)} \mathbb{E}_P [ \gamma(x)] - \mathbb{E}_Q [ \gamma(y)] \,. 
\end{align}
With this, the following primal-dual objective can be used to concurrently find the best-performing classifier and representation mapping, 
\begin{align}\label{eq:optobj}
    &\min_\theta \max_\omega \mathbb{E}\left[\ell_{h_{\theta,T}}\left(g_\omega(X_{T}), Y_T\right) \mid Z_{1}^{T-1}\right] + \lambda L_D \,, 
\end{align}
where $L_D = \max_{t
}\mathbb{E}_{P_t}\left[g_\omega(X_{t})\right] - \mathbb{E}_{P_{t+1}}\left[g_\omega(X_{t+1})\right] $ and $\lambda$ is a tunable parameter.

\paragraph{One-step and temporal variants}
Notice that $L_D$ relies on the maximum distance across all domains. It is thus hard to directly evaluate $L_D$ without simultaneous access to all domains. With access only to the current and the past domains, we could optimize the following one-step primal-dual loss at time $t$ instead.
\begin{align} \label{eq:onesteploss}
    &\min_\theta \max_\omega \mathbb{E}\left[\ell_{h_{\theta,t}}\left( g_\omega(X_{t}), Y_t\right) \mid Z_{1}^{t}\right] + \lambda L_{D_t} \,,
\end{align}
where $L_{D_t} = \mathbb{E}_{P_t}\left[g_\omega(X_{t})\right] - \mathbb{E}_{P_{t+1}}\left[g_\omega(X_{t+1})\right] $. 

Compared to the objective (\ref{eq:optobj}), the one-step loss (\ref{eq:onesteploss}) only gives us partial information, and directly optimizing it may often lead to suboptimal performance. While it is inevitable to optimize with some loss of information under the problem setup, we use a temporal model (like an LSTM) to help preserve historical data information in the process of learning mapping function $g$. In particular, in the temporal variant, we will be using the hidden states of an LSTM to dynamically summarize the features from all the past domains. Then, we shall use the feature distribution computed from the LSTM hidden state to align with the feature distribution at the current time step. 

To practically implement these objectives, we can use neural networks to learn the representation and the classifier, and another neural network is used as a critic to judge the quality of the learned representations. To minimize the distance between representations of different domains, one can use $W_1$ distance as an empirical metric. Then the distance of the critic of the representations in different domains is then minimized to encourage the learning of similar representations. We note that the use of $W_1$ distance, which is easy to evaluate empirically, to guide representation learning has been practiced before~\citep{shen2018wasserstein}. We take this approach further to the problem of gradual domain adaption. 

\section{Empirical Results}
In this section, we perform experiments to demonstrate the effectiveness of supervised gradual domain adaptation and compare our algorithm with No Adaptation, Direct Adaptation, and Multiple Source Domain Adaptation (MDA) on different datasets. We also verify the insights we obtained in the previous section by answering the following three questions: 
\begin{enumerate}
    \item \textbf{How helpful is representation learning in gradual domain adaptation?} Theoretically, effective representation where the data drift is ``small'' helps algorithms to gradually adapt to the evolving domains. This corresponds to minimizing the $\rho \Delta$ term in our Theorem \ref{thm:main}. We show that our algorithm with objective (\ref{eq:onesteploss}) outperforms the objective of empirical risk (No Adaptation). 
    \item \textbf{Can the one-step primal-dual loss (\ref{eq:onesteploss}) act as an substitute to optimization objective (\ref{eq:optobj})?} Inspired by our theoretical results (Theorem \ref{thm:main}), the primal-dual optimization objective (\ref{eq:optobj}) should guide the adaptation process. However, optimization of this objective requires simultaneous access to all data domains. We use a temporal encoding (through a temporal model such as LSTM) of historical data to demonstrate the importance of the information of past data domains. We compare this to results obtained with a convolutional network (CNN)-based model to verify that optimizing the one-step loss (\ref{eq:onesteploss}) with a temporal model could largely mitigate the information loss. 
    \item \textbf{Does the length of gradual domain adaptation affect the model's ability to adapt?} Our theoretical results suggest that there exists an optimal length $T$ for gradual domain adaptation. Our empirical results corroborate this as when the time horizon passes a certain threshold the model performance is saturated.
\end{enumerate}

\subsection{Experimental Setting}
We conduct our experiments on Rotating MNIST, Portraits, and FMOW, with a detailed description of each dataset in the appendix. We compare the performance of no adaptation, direct adaptation, and multiple source domain adaptations with gradual adaptation. The implementation of each method is also included in the appendix. Each experiment is repeated over 5 random seeds and reported with the mean and $1$ std.

\subsection{Experimental Results}
\begin{table*}[ht]\centering
\captionsetup{justification=centerlast}
\caption{Results on rotating MNIST dataset with Gradual Adaptation on 5 domains, Direct Adaptation, and No Adaptation. }\label{table:rotate}
\begin{tabular}{@{}cccccc@{}}
\toprule
\multicolumn{6}{c}{Rotating MNIST with 5 domains}                                                                                                                                                                                   \\ \midrule
\multicolumn{1}{c|}{}             & \multicolumn{2}{c|}{Gradual Adaptation}                         & \multicolumn{2}{c|}{Direct Adaptation}                                        & No Adaptation    \\ \midrule
\multicolumn{1}{c|}{}             & \multicolumn{1}{c|}{CNN}              & \multicolumn{1}{c|}{LSTM}             & \multicolumn{1}{c|}{CNN}              & \multicolumn{1}{c|}{LSTM}             & CNN              \\ \midrule
\multicolumn{1}{c|}{0-30 degree}  & \multicolumn{1}{c|}{90.21 $\pm$ 0.48} & \multicolumn{1}{c|}{\textbf{94.83} $\pm$ 0.49} & \multicolumn{1}{c|}{77.97 $\pm$ 0.99} & \multicolumn{1}{c|}{89.72 $\pm$ 0.73} & 79.76 $\pm$ 3.20 \\ \midrule
\multicolumn{1}{c|}{0-60 degree}  & \multicolumn{1}{c|}{87.35 $\pm$ 1.02} & \multicolumn{1}{c|}{\textbf{92.52} $\pm$ 0.25} & \multicolumn{1}{c|}{73.27 $\pm$ 1.51} & \multicolumn{1}{c|}{88.53 $\pm$ 0.76} & 58.36 $\pm$ 2.59 \\ \midrule
\multicolumn{1}{c|}{0-120 degree} & \multicolumn{1}{c|}{82.38 $\pm$ 0.57} & \multicolumn{1}{c|}{\textbf{89.72} $\pm$ 0.35} & \multicolumn{1}{c|}{62.52 $\pm$ 1.06} & \multicolumn{1}{c|}{84.30 $\pm$ 2.60} & 38.25 $\pm$ 0.61 \\ \bottomrule
\end{tabular}
\end{table*}


\begin{table*}[ht]\centering
\captionsetup{justification=centerlast}
\caption{Results on rotating MNIST dataset (5 domains) with Gradual Adaptation, MDA (MDAN)~\citep{zhao2018adversarial} and Meta-learning (EAML)~\citep{liu2020learning}. }\label{table:rotate2}
\scalebox{0.9}{\begin{tabular}{@{}ccccccc@{}}
\toprule
\multicolumn{7}{c}{Rotating MNIST with 5 domains}                                                                                                                                                                                                                                                                      \\ \midrule
\multicolumn{1}{c|}{}             & \multicolumn{2}{c|}{Gradual Adaptation}                                       & \multicolumn{3}{c|}{MDAN}                                                                                                                                                      & EAML              \\ \cmidrule(l){2-7} 
\multicolumn{1}{c|}{}             & \multicolumn{1}{c|}{CNN}              & \multicolumn{1}{c|}{LSTM}             & \multicolumn{1}{c|}{Maxmin}           & \multicolumn{1}{c|}{Dynamic}          & \multicolumn{1}{c|}{\begin{tabular}[c]{@{}c@{}}Dynamic \\ with last \\ 2 domains\end{tabular}} & CNN               \\ \midrule
\multicolumn{1}{c|}{0-30 degree}  & \multicolumn{1}{c|}{90.21 $\pm$ 0.48} & \multicolumn{1}{c|}{9.83 $\pm$ 0.49} & \multicolumn{1}{c|}{93.62 $\pm$ 0.87} & \multicolumn{1}{c|}{\textbf{95.79} $\pm$ \textbf{0.33}} & \multicolumn{1}{c|}{83.04 $\pm$ 0.29}                                                          & 79.10 $\pm$  6.99  \\ \midrule
\multicolumn{1}{c|}{0-60 degree}  & \multicolumn{1}{c|}{87.35 $\pm$ 1.02} & \multicolumn{1}{c|}{\textbf{92.52} $\pm$ \textbf{0.25}} & \multicolumn{1}{c|}{91.99 $\pm$ 0.51} & \multicolumn{1}{c|}{92.27 $\pm$ 0.26} & \multicolumn{1}{c|}{61.49 $\pm$ 0.72}                                                          & 54.64 $\pm$  4.08  \\ \midrule
\multicolumn{1}{c|}{0-120 degree} & \multicolumn{1}{c|}{82.38 $\pm$ 0.57} & \multicolumn{1}{c|}{\textbf{89.72} $\pm$ \textbf{0.35}} & \multicolumn{1}{c|}{87.25 $\pm$ 0.52} & \multicolumn{1}{c|}{88.57 $\pm$ 0.21} & \multicolumn{1}{c|}{41.14 $\pm$ 1.77}                                                          & 17.91 $\pm$ 10.34 \\ \bottomrule
\end{tabular}}
\end{table*}

\begin{table*}[ht]\centering
\caption{Results on FMOW with Gradual Adaptation with 3 domains, Direct Adaptation, and No Adaptation. }\label{table:fmow}
\begin{tabular}{@{}clcc@{}}
\toprule
\multicolumn{4}{c}{FMOW}                                      \\ \midrule
\multicolumn{1}{c|}{\begin{tabular}[c]{@{}c@{}}No Adaptation\\ with ERM\end{tabular}} & \multicolumn{1}{c|}{\begin{tabular}[c]{@{}c@{}}Direct Adaptation\\ with CNN\end{tabular}} & \multicolumn{1}{c|}{\begin{tabular}[c]{@{}c@{}}Gradual Adaptation \\ with CNN\end{tabular}} & \begin{tabular}[c]{@{}c@{}}Gradual Adaptation\\ with LSTM\end{tabular} \\ \midrule
\multicolumn{1}{c|}{33.10 $\pm$ 1.94}& \multicolumn{1}{c|}{41.94 $\pm$ 2.73}& \multicolumn{1}{c|}{36.86 $\pm$ 1.91}& 
\textbf{43.52} $\pm$ 1.40                                              \\ \bottomrule
\end{tabular}
\end{table*}

\paragraph{Learning representations further help in gradual adaptation}~~

On rotating MNIST, the performance of the model is better in most cases when adaptation is considered (Table \ref{table:rotate}), which demonstrates the benefit of learning proper representations. With a CNN architecture, the only exception is when the shift in the domain is relatively small ($0$ to $30$ degree), where the No Adaptation method achieves higher accuracy than the Direct Adaptation method by $2\%$. However, when the shift in domains is relatively large, Adaptation methods are shown to be more successful in this case and this subtle advantage of No Adaptation no longer holds. Furthermore, Gradual Adaptation further enhances this outperformance significantly. This observation shows the advantage of sequential adaptation versus direct adaptation.
\begin{table}[htb]
\centering
\caption{Results on Portraits with Gradual Adaptation for different lengths of horizon $T$, Direct Adaptation, and No Adaptation. }\label{table:port}
\begin{tabular}{@{}ccc@{}}
\toprule
\multicolumn{3}{c}{Portraits}                                                                        \\ \midrule
\multicolumn{1}{c|}{}                     & \multicolumn{1}{c|}{CNN}              & LSTM             \\ \midrule
\multicolumn{1}{c|}{No Adaptation}         & \multicolumn{1}{c|}{76.01 $\pm$ 1.45} &  N/A                \\ \midrule
\multicolumn{1}{c|}{Direct Adaptation}    & \multicolumn{1}{c|}{86.86 $\pm$ 0.84} &   N/A               \\ \midrule
\multicolumn{1}{c|}{Gradual - 5 Domains}  & \multicolumn{1}{c|}{87.77 $\pm$ 0.98} & 87.41 $\pm$ 0.76 \\ \midrule
\multicolumn{1}{c|}{Gradual - 7 Domains}  & \multicolumn{1}{c|}{89.14 $\pm$ 1.64} & 89.15 $\pm$ 1.12 \\ \midrule
\multicolumn{1}{c|}{Gradual - 9 Domains}  & \multicolumn{1}{c|}{90.46 $\pm$ 0.54} 
& 89.88 $\pm$ 0.54 \\ \midrule
\multicolumn{1}{c|}{Gradual - 11 Domains} & \multicolumn{1}{c|}{90.56 $\pm$ 1.21} 
& 90.93 $\pm$ 0.75 \\ \midrule 
\multicolumn{1}{c|}{Gradual - 12 Domains} 
& \multicolumn{1}{c|}{91.45 $\pm$ 0.27} 
& 90.73 $\pm$ 0.66 \\\midrule
\multicolumn{1}{c|}{Gradual - 13 Domains} & \multicolumn{1}{c|}{ 90.54 $\pm$ 0.90} 
&	91.13 $\pm$ 0.35 \\\midrule
\multicolumn{1}{c|}{Gradual - 14 Domains} & \multicolumn{1}{c|}{90.58 $\pm$ 0.38} & 90.35 $\pm$ 0.71 \\\bottomrule
\end{tabular}
\end{table}
We further show that the performance of the algorithm monotonically increases as it progress to adapt to each domain and learn a cross-domain representation. Figure \ref{fig:gradual} shows the trend in algorithm performance on rotating MNIST and FMOW.

\paragraph{One-step loss is insufficient as a substitute, but can be improved by temporal model}~~
The inefficiency of adaptation without historical information appears with all datasets we have considered, reflected through Table \ref{table:rotate}, \ref{table:fmow}, \ref{table:port}. In almost all cases, we observe that learning with a temporal model (LSTM) achieves better accuracy than a convolutional model (CNN). 
The gap is especially large on FMOW, the large-scale dataset in our experiments. We suspect that optimizing with only partial information can lead to suboptimal performance on such a complicated task. This is reflected through the better performance achieved by Direct Adaptation with CNN when compared to Gradual Adaptation with CNN and 3 domains (Table \ref{table:fmow}). In contrast, Gradual Adaptation with LSTM overtakes the performance of Direct Adaptation, suggesting the importance of historical representation. Another evidence is that Figure \ref{fig:gradual} shows that Gradual Adaptation with a temporal model performs better on all indexes of domains on rotating MNIST and FMOW.

\paragraph{Existence of optimal time horizon}~~
With the Portraits dataset and different lengths of horizon $T$, we verify that then optimal time horizon can be reached when model performance is saturated in Table \ref{table:port}. The performance of the model increases drastically when the shifts in domains are considered, shown by the difference in the performance of No Adaptation, Direct Adaptation, and Gradual Adaptation with $5$ and $7$ domains. However, this increase in performance becomes relatively negligible when $T$ is large (the performance gain is saturated when the horizon length is $9 - 14$). This rate of growth in accuracy implies that there exists an optimal number of domains.

\begin{figure}
    \centering
    \begin{subfigure}[b]{0.35\textwidth}
         \centering
         \includegraphics[width=\textwidth]{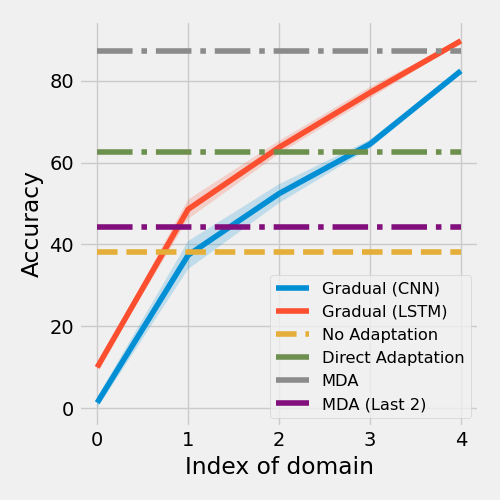}
         \caption{Rotating MNIST}\label{fig:rotate120}
     \end{subfigure}
     \begin{subfigure}[b]{0.35\textwidth}
         \centering
         \includegraphics[width=\textwidth]{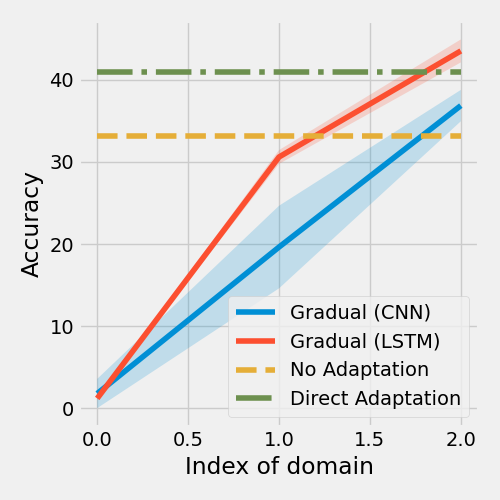}
         \caption{FMOW}\label{fig:gradual}
     \end{subfigure}
    \begin{subfigure}[b]{0.6\textwidth}
         \centering
         \includegraphics[width=\textwidth]{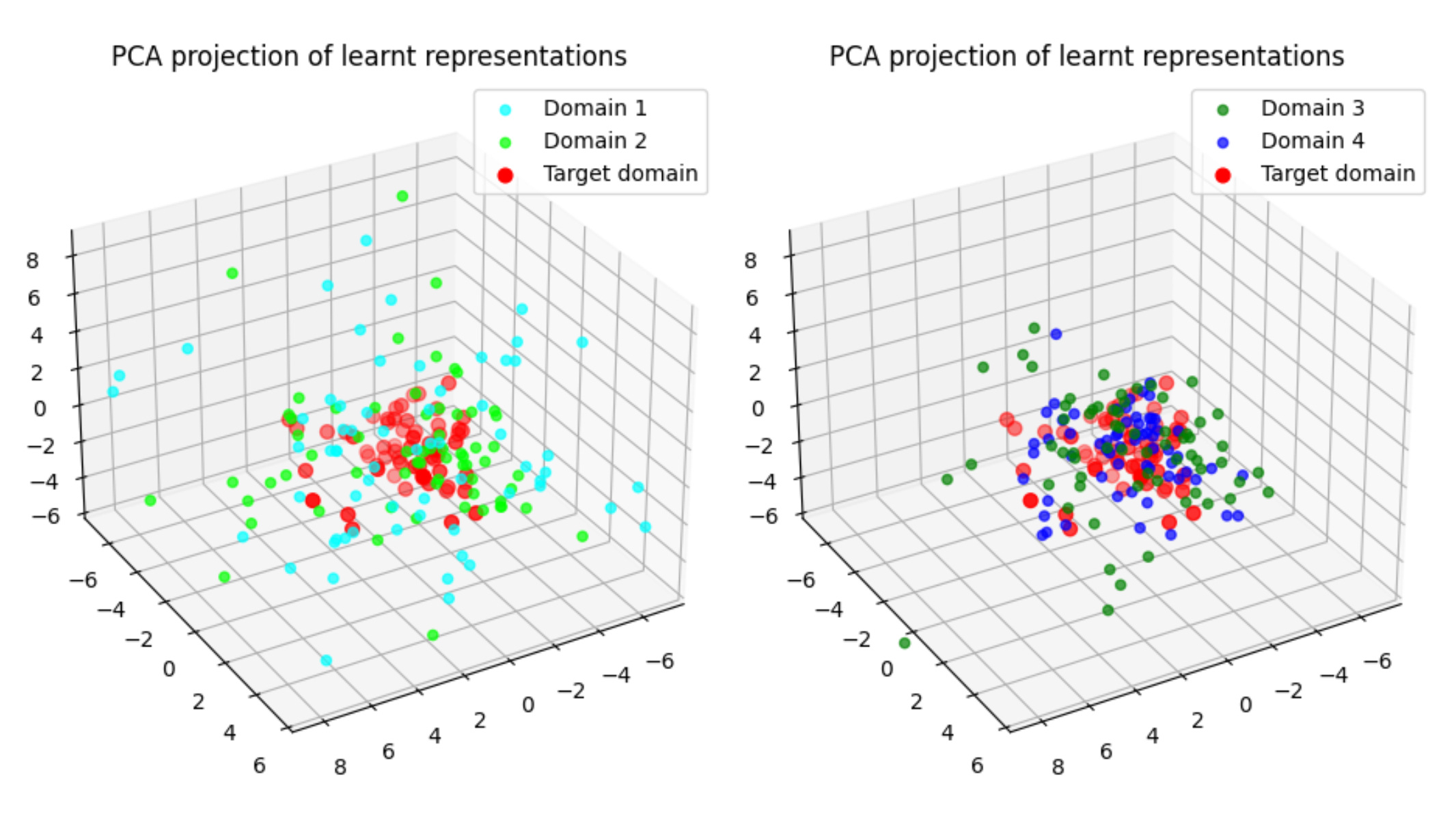}
         \caption{PCA projection}\label{fig:pca}
     \end{subfigure}
     \begin{subfigure}[b]{0.35\textwidth}
         \centering
         \includegraphics[width=\textwidth]{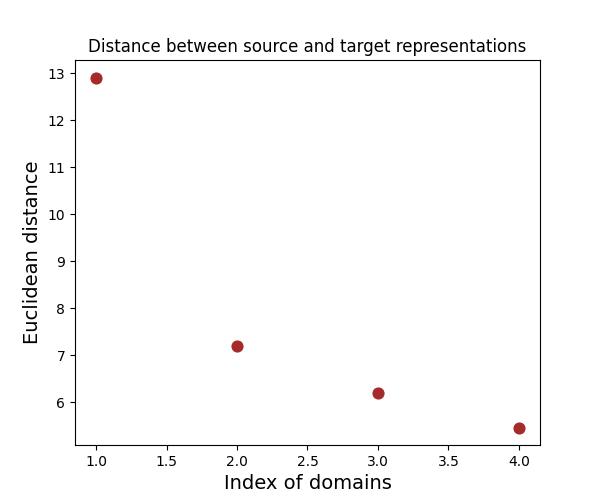}
         \caption{Plot of euclidean distance}\label{fig:eu}
     \end{subfigure} 
     \caption{Figure \ref{fig:rotate120} compares the training curves on rotating MNIST with maximum rotation of 120 degrees. Figure \ref{fig:gradual} compares the training curves on FMOW. Figure \ref{fig:pca} is the PCA projection plot of learned representation and Figure \ref{fig:eu} plots the Euclidean distance to the target domain of the projections of learned representations.}
\end{figure}
\paragraph{Comparison with MDA}~~ Lastly, we remark on the results (Table \ref{table:rotate2} and \ref{new_result}) achieved by Gradual Adaptation in comparison with MDA methods (MDAN \citep{zhao2018adversarial}, DARN \citep{wen2020domain} and Fish \citep{shi2022gradient}). On Rotating MNIST, we note that Gradual Adaptation outperforms MDA methods when the shift is large ($60$ and $120$ degree rotation) while relaxing the requirement of simultaneous access to all source domains. It is only when the shift is relatively small ($30$-degree rotation), MDA method DARN achieves a better result than ours. When the MDA method is only presented with the last two training domains, Gradual Adaptation offers noticeable advantages regardless of the shift in the domains (Table \ref{table:rotate2}). This demonstrates the potential of graduate domain adaptation in real applications that even when the data are not simultaneously presented it is possible to achieve competitive or even better performance. 
\begin{table}[htb]
\centering
\caption{Results on rotating MNIST dataset with Gradual Adaptation on 5 domains and MDA methods, Fish \citep{shi2022gradient} and DARN \citep{wen2020domain} }\label{new_result}
\begin{tabular}{@{}clcc@{}}
\toprule
& \multicolumn{1}{c}{Fish}  & DARN                 & Ours                      \\ \midrule
\multicolumn{1}{c|}{0-30 degree}  & \textbf{95.83} $\pm$ \textbf{0.13} & 94.20 $\pm$ 0.27  & 94.83 $\pm$ 0.49          \\
\multicolumn{1}{c|}{0-60 degree}  & 90.57 $\pm$ 0.37          & 89.50 $\pm$ 0.12 & \textbf{92.52} $\pm$ \textbf{0.25} \\
\multicolumn{1}{c|}{0-120 degree} & 83.26 $\pm$ 1.58          & 82.28 $\pm$ 2.42 & \textbf{89.72} $\pm$ \textbf{0.35} \\ \bottomrule
\end{tabular}
\end{table}
One possible reason for this can be illustrated by Figure \ref{fig:eu}, in which we plot the PCA projections and the Euclidean distance to the target domain of learned representations. From Figure \ref{fig:eu}, we can see that the gradual domain adaptation method is able to gradually learn an increasingly closer representation of the source domain to the target domain. This helps our method to make our prediction based on more relevant features while MDA methods may be hindered by not-so-relevant features from multiple domains.

\paragraph{Comparison with Meta-learning}
In addition to comparing our method with those achieved by domain adaptation methods, we also compare our work with Meta-learning for evolving distributions (EAML \citep{liu2020learning}). In contrast to gradual domain adaptation, these methods instead assume that the testing distribution is gradually evolving and thus the learned classifier is desired to be adaptive to these changing targets. In addition, the learned classifier is also asked to ``not forget''. Due to the different objectives, we can see that the baseline method (EAML) is suboptimal on a fixed target with changing training distributions, when compared to our method (Table \ref{table:rotate2}).
\section{Conclusion}
We studied the problem of supervised gradual domain adaptation, which arises naturally in applications with temporal nature. In this setting, 
we provide the first learning bound of the problem and our results are general to a range of loss functions and are algorithm agnostic. 
Based on the theoretical insight offered by our theorem, we designed a primal-dual learning objective to learn an effective representation across domains while learning a classifier. 
We analyze the implications of our results through experiments on a wide range of datasets. 
\section*{Acknowledgements}
We want to thank the reviewers and the editors for their constructive comments during the review process.  
Jing Dong and Baoxiang Wang are partially supported by National Natural Science Foundation of China (62106213, 72150002) and Shenzhen Science and Technology Program (RCBS20210609104356063, JCYJ20210324120011032). Han Zhao would like to thank the support from a Facebook research award.

\bibliography{reference}

\newpage
\appendix

\section{Proof of Lemma~\ref{lem:loss_wp_bound}} 
\losswpbound*
\begin{proof}
Let $\|f\|_{\text{Lip}}$ denotes the minimum Lipschitz constant of $f$, then
\begin{align*}
    \Exp_\mu[\ell_h(X, Y)] - \Exp_\nu[\ell_h(X', Y')] 
    = \ &  \int \ell_h(X, Y)~d\mu - \int \ell_h(X', Y')~d\nu \\
    \leq \ & \sup_{\|f\|_{\text{Lip}} \leq \rho}\int f(X, Y)~d\mu - \int f(X, Y)~d\nu \tag{$\ell(\cdot,\cdot)$ is $\rho$-Lipschitz continuous}\\
    = \ & \rho\cdot \sup_{\|f\|_{\text{Lip}} \leq 1}\int f(X, Y)~d\mu - \int f(X, Y)~d\nu \\
    \leq \ & \rho W_1(\mu,\nu) \tag{Kantorovich-Rubinstein Duality of $W_1$}\\
    \leq \ & \rho W_p(\mu,\nu) \tag{Monotonicity of $W_p$}\\ 
    \leq \ & \rho\Delta \,,
\end{align*}

where the second to last inequality is due to the monotonicity of the $W_p(\cdot,\cdot)$ metric, i.e. for $1\leq p \leq q$, $W_p (\cdot, \cdot) \leq W_q ( \cdot, \cdot)$. 

To see this, we first define $\beta(t) := |t|^{q/p} $ and notice that it is convex in $t$. Then, let $\Gamma(\mu, \nu)$ be the set of all couplings between $\mu, \nu$ and $\gamma$ be any coupling between $\mu$ and $\nu$, we have
\begin{align*}
    \left( \int \| x - y\|^p d \gamma \right)^{1/p} &= \left(\Exp_\gamma[\|x - y\|^p]\right)^{1/p} \\
    &= \beta\left(\Exp_\gamma[\|x - y\|^p]\right)^{1/q} \\
    &\leq \left(\Exp_\gamma[\beta\left(\|x - y\|^p]\right)\right)^{1/q} \\
    &= \left(\Exp_\gamma[\|x - y\|^q]\right)^{1/q} \\
    &= \left( \int \| x - y\|^q d \gamma \right)^{1/q},
\end{align*}
where $x \sim \mu, y \sim \nu$ and the inequality is due to Jensen's inequality since $\beta(\cdot)$ is convex. As the above inequality holds for every coupling $\gamma\in\Gamma(\mu, \nu)$, it follows that 
\begin{equation*}
    W_p(\mu,\nu) = \inf_{\gamma\in\Gamma(\mu,\nu)}\left( \int \| x - y\|^p d \gamma \right)^{1/p}\leq \inf_{\gamma\in\Gamma(\mu,\nu)}\left( \int \| x - y\|^q d \gamma \right)^{1/q} = W_q(\mu,\nu),
\end{equation*}
which implies the monotonicity of $W_p(\cdot,\cdot)$.
\end{proof}

\section{Proof of Theorem~\ref{thm:main}} 
To simplify the notation, we let $f(x, y) = \ell(h(x), y)$ and denote the function class of $f$ as $\mathcal{F}$. We first state a few results from non-stationary time series analysis \citep{kuznetsov2020discrepancy}, which are used in our theorem. 
\begin{lem}[Corollary 2 \citet{kuznetsov2020discrepancy}] \label{lem:cor2}
For any $\delta>0$, with probability at least $1-\delta$, 
\begin{align*}
\mathbb{E}\left[f\left(Z_{T+1}\right) \mid Z_{1}^{T}\right] 
\leq \sum_{t=1}^{T} \frac{1}{T} f\left(Z_{t}\right)+\operatorname{disc}_T+ \frac{1}{T} +6 M \sqrt{4 \pi \log T \Re_{T-1}^{s e q}(\mathcal{F})}+\frac{M}{T}  \sqrt{8 \log \frac{1}{\delta}} \,,
\end{align*}
where $\operatorname{disc}_T$ is the discrepancy measure defined in Definition \ref{def:discrepancy} and $\mathfrak{R}_{T}^{s e q}(\mathcal{F})$ is the sequential Rademacher complexity defined in Definition~\ref{def:seq_complexity}.
\end{lem}

\begin{restatable}{lem}{corthree}[Corollary 3 of \citet{kuznetsov2020discrepancy}]\label{lem:cor3}
Let $h_{0}=\operatorname{argmin}_{h \in \mathcal{H}} \frac{1}{T}\sum_{t=1}^{T} \ell\left(h(X_t), Y_{t}\right)$ and $q = (q_1, \ldots, q_T)$ be real numbers. 
Let $q = (q_1, \ldots, q_T)$ be real numbers. For any $\delta>0$, with probability at least $1-\delta$, 
\begin{align}
\mathbb{E}\left[f_{0}\left(Z_{T}\right) \mid Z_{1}^{T-1}\right] 
&\leq \inf _{f \in \mathcal{F}} \mathbb{E}\left[f\left(Z_{T}\right) \mid Z_{1}^{T-1}\right]+\operatorname{disc}_T+ \frac{1}{T} +6 M \sqrt{4 \pi \log T} \mathfrak{R}_{T-1}^{s e q}(\mathcal{F})+\frac{M}{T} \sqrt{8 \log \frac{1}{\delta}} \,,
\end{align}
where $\operatorname{disc}_T$ is the discrepancy measure defined in Definition~\ref{def:discrepancy} and $\mathfrak{R}_{T}^{s e q}(\mathcal{F})$ is the sequential Rademacher complexity  defined in Definition~\ref{def:seq_complexity}.
\end{restatable}

\mainthm*

\begin{proof}
With Lemma \ref{lem:loss_wp_bound}, we first show that the discrepancy measure can also be bounded by $\rho \Delta$, the maximum Wasserstein distance between any two consecutive data distributions. We start by decomposing the discrepancy measure utilizing an adjustable summation $\frac{1}{s} \sum_{t=T-s+1}^{T} \mathbb{E}\left[f\left(Z_{t}\right) \mid Z_{1}^{t-1}\right]$, 
\begin{align*}
\operatorname{disc}_T \leq & \sup_{f \in \mathcal{F}}\left(\frac{1}{s} \sum_{t=T-s+1}^{T} \mathbb{E}\left[f\left(Z_{t}\right) \mid Z_{1}^{t-1}\right]- \frac{1}{T} \sum_{t=1}^{T}  \mathbb{E}\left[f\left(Z_{t}\right) \mid Z_{1}^{t-1}\right]\right) \\
&+\sup_{f \in \mathcal{F}}\left(\mathbb{E}\left[f\left(Z_{T}\right) \mid Z_{1}^{T-1}\right]-\frac{1}{s} \sum_{t=T-s+1}^{T} \mathbb{E}\left[f\left(Z_{t}\right) \mid Z_{1}^{t-1}\right]\right) \,.
\end{align*}
Notice that the second term can be further decomposed as
\begin{align*}
    &\sup_{f \in \mathcal{F}}\left(\mathbb{E}\left[f\left(Z_{T}\right) \mid Z_{1}^{T-1}\right]-\frac{1}{s} \sum_{t=T-s+1}^{T} \mathbb{E}\left[f\left(Z_{t}\right) \mid Z_{1}^{t-1}\right]\right) \\
    \leq & \frac{1}{s} \sum^{T-s+1}_{t=1}\sup _{f \in \mathcal{F}}\left(\mathbb{E}\left[f(Z_{T}) \mid Z_{1}^{T-1}\right]-\mathbb{E}\left[f\left(Z_{t}\right) \mid Z_{1}^{t-1}\right]\right) \\
    \leq & \frac{1}{s} \sum^{T-s+1}_{t=1}\sum_{r=t}^{T} \sup _{f \in \mathcal{F}}\left(\mathbb{E}\left[f\left(Z_{r+1}\right) \mid Z_{1}^{r}\right]-\mathbb{E}\left[f\left(Z_{r}\right) \mid Z_{1}^{r-1}\right]\right)\,.
\end{align*}
As the supremum is taken over $f \in \mathcal{F}$, this cumulative difference term can be understood as comparing the loss caused by the same classifier along the trajectory across the time horizon. As we have shown in Lemma \ref{lem:loss_wp_bound}, for any classifier, the loss incurred on two consecutive distribution of bounded $W_p$ distance of $\Delta$ is at most $\rho \Delta$. Thus we have $\operatorname{disc}_T \leq T\rho \Delta$. 

Armed with the bounded discrepancy measure, we are now ready to bound the learning error in interest. Let us first define $f_{0}=\operatorname{argmin}_{f \in \mathcal{F}} \sum_{t=1}^{T} \frac{1}{T} f\left(Z_{t}\right)$ and $\Phi\left(Z_{1}^{T}\right)=\sup _{f \in \mathcal{F}}\left(\mathbb{E}\left[f\left(Z_{T}\right) \mid Z_{1}^{T-1}\right] \right.$ $\left.-\sum_{t=1}^{T} \frac{1}{T} f\left(Z_{t}\right)\right)$. Then the conditional expected difference between $\ell_{h_{T}}\left(X_{T}, Y_{T}\right)$ and $f_{0}$ can be upper bounded by $\Phi\left(Z_{1}^{T}\right)$.
\begin{align*}
    &\mathbb{E}\left[\ell_{h_{T}}\left(X_{T}, Y_{T}\right)\mid Z_{1}^{T-1}\right]- \mathbb{E}\left[f_{0}\left(Z_{T}\right)\mid Z_1^{T-1}\right] \\
    =& \left(\mathbb{E}\left[\ell_{h_{T}}\left(X_{T}, Y_{T}\right) \mid Z_{1}^{T-1}\right] - \frac{1}{T} \sum^{T-1}_{t=1} \left[\ell_{h_{T}}\left(X_{t}, Y_{t}\right)\right] \right)
    + \left(  \frac{1}{T}\sum^{T-1}_{t=1} \left[\ell_{h_{T}}\left(X_{t}, Y_{t}\right)\right] -   \frac{1}{T}\sum^{T-1}_{t=1}f_{0}\left(Z_{t} \right) \right) \\
    +& \left(  \frac{1}{T}\sum^{T-1}_{t=1}f_{0}\left(Z_{t} \right) - \mathbb{E}\left[f_{0}\left(Z_{T+1}\right)\mid Z_1^{T-1}\right] \right)\\
    \leq & 2\Phi(Z_1^T) + \left(  \frac{1}{T}\sum^{T-1}_{t=1} \left[\ell_{h_{T}}\left(X_{t}, Y_{t}\right)\right] - \frac{1}{T}\sum^{T-1}_{t=1}f_{0}\left(Z_{t} \right) \right)\,.
\end{align*}
By Lemma \ref{lem:cor2}, for any $\delta>0$, with probability at least $1-\delta$, the following inequality holds for any $f$ 
\begin{align*}
\mathbb{E}\left[f\left(Z_{T}\right) \mid Z_{1}^{T-1}\right] 
&\leq \frac{1}{T}\sum_{t=1}^{T-1}  f\left(Z_{t}\right)+\operatorname{disc}_T+ \frac{1}{T} +6 M \sqrt{4 \pi \log T} \mathcal{R}_{T-1}^{s e q}(\mathcal{F})+\frac{M}{T} \sqrt{8 \log \frac{1}{\delta}} \,.
\end{align*}
Rearranging the terms, substituting the upper bound of $\operatorname{disc}_T$ and hence we have that
\begin{align*}
    \Phi\left(Z_{1}^{T}\right) \leq T\rho \Delta+ \frac{1}{T} +6 M \sqrt{4 \pi \log T} \mathcal{R}_{T-1}^{s e q}(\mathcal{F})+ \frac{M}{T} \sqrt{8 \log \frac{1}{\delta}} \,.
\end{align*}
It remains to bound $  \frac{1}{T}\sum^{T-1}_{t=1} \left[\ell_{h_{T}}\left(X_{t}, Y_{t}\right)\right] - \frac{1}{T}\sum^{T-1}_{t=1}f_{0}\left(Z_{t} \right)$, where we leverage classical results from online learning literature. Let $\hat{h}_{t}$ be the classifier picked by an optimal online learning algorithm at time $t$. Then with a total $nT$ number of data points, the optimal algorithm will suffer an additional loss of at most $\mathcal{O}\sqrt{nT}$ compared to the optimal loss. Thus we have 
\begin{align*}
    \frac{1}{T}\sum^{T-1}_{t=1}\mathbb{E}\left[\ell_{\hat{h}_{t}}\left(X_{t}, Y_{t}\right)\right]- \frac{1}{T}\sum^{T-1}_{t=1}f_{0}\left(Z_{t} \right)=O\left(\frac{1}{\sqrt{nT}}\right) \,.
\end{align*}
Further, the difference between our classifier and the classifier picked by the online learning algorithm $ \sum^{T-1}_{t=1} \left[\ell_{h_{T}}\left(X_{t}, Y_{t} \right)\right] -  \sum^{T-1}_{t=1}  \mathbb{E}\left[\ell_{\hat{h}_{t}}\left(X_{t}, Y_{t}\right)\right]$ can be upper bounded by the VC inequality (Section 3.5 \citep{bousquet2004introduction}) as
\begin{align*}
     \frac{1}{T}\sum^{T-1}_{t=1} \left( \mathbb{E}\left[\ell_{h_{T}}\left(X_{t}, Y_{t} \right) \right] - \mathbb{E}\left[\ell_{\hat{h}_{t}}\left(X_{t}, Y_{t}\right)\right] \right)
    \leq & \frac{1}{T} \sum^{T-1}_{t=1} \left( \mathbb{E}\left[\ell_{h_{T}}\left(X_{t}, Y_{t} \right) \right] - \min_{\Tilde{h}}\mathbb{E}\left[\ell_{\Tilde{h}}\left(X_{t}, Y_{t} \right) \right] \right) \\
    \leq & \frac{1}{T}\sqrt{\frac{\text{VCdim}(\mathcal{H}) + \log (2/\delta)}{2n}}\,.
\end{align*}
Lastly by Lemma \ref{lem:cor3} and combining the terms, we have the final bound in the theorem 
\begin{align*}
\mathbb{E}\left[\ell_{h_T}\left(X_T, Y_{T}\right) \mid Z_{1}^{T-1}\right] 
\leq & \mathbb{E}\left[\ell_{h_0}\left(X_T, Y_T\right) \mid Z_{1}^{T-1}\right] + \frac{3}{T} +\frac{3M}{T} \sqrt{8 \log \frac{1}{\delta}} + \frac{1}{T}\sqrt{\frac{\text{VCdim}(\mathcal{H}) + \log (2/\delta)}{2n}} \\
&+ O\left(\frac{1}{\sqrt{nT}}\right) + 18 M \sqrt{4 \pi \log T} \mathfrak{R}_{T-1}^{s e q}(\mathcal{F})+ 3T\rho\Delta\,.\qedhere
\end{align*}
\end{proof}

\section{Experimental set up}
\subsection{Dataset}
\textbf{Rotating MNIST}~~
This is a semi-synthetic dataset from the MNIST dataset. We rotate the image continuously from $0-30$, $0-60$, and $0-120$ degrees across the time horizon, forming 3 datasets for our gradual adaptation task. We create each dataset in a way such that the degree of rotation increases linearly as $t$ increases. Each dataset is then partitioned into more domains, e.g. 0-30 degrees rotations into 5 domains would give domains each with images of $0-5$, $6-12$, $13-18$, $19-24$, $25-30$ degrees rotations respectively. 

\textbf{Portraits}~~This dataset contains $37921$ portraits of high school seniors across years (1910 - 2010) \citep{ginosar2015century}. The data is then partitioned into domains by their year index. According to the specified number of domains, the partition is done so evenly by years (e.g. for a 5 domains dataset, the split is so that the domains has images from the year 1910-1930, 1930-1950, 1950-1970, 1970-1990, 1990-2010 respectively). 

\textbf{FMOW}~~This dataset is composed of over 1 million satellite images and their building/land use labels from 62 categories \citep{christie2018functional,koh2021wilds} from $2002$-$2017$. The input is an RGB image of $224 \times 224$ pixels and each image comes with metadata of the year it is taken. The data is split into domains chronologically and the target domain is the year $2017$. Specifically, the data is split similarly to the splitting used for the Portraits dataset, where the images are separated into domains according to their year. 

\subsection{Algorithms and model architecture}
\textbf{No Adaptation}~~
For this method, we perform empirical risk minimization with cross-entropy loss on the initial domain and then test on the target domain. To extract the feature information, we use model VGG16 \citep{simonyan2014very} for MNIST/Portrait, and ResNet18 \citep{he2016deep} for FMOW. Then a two-layer MLP is used as the classifier. 

\textbf{Direct Adaptation}~~ 
We group the training domains $t = 1, \ldots, T - 1$ and let the algorithm learn to adapt from the grouped domain to the target domain. 
We use cross-entropy loss with objective (\ref{eq:onesteploss}). To extract the features,  we used VGG16 \citep{simonyan2014very} for MNIST/Portrait and ResNet18 \citep{he2016deep} for FMOW. To test the effectiveness of the encoding historical representation, we use a 2-layer GRU for MNIST/Portrait and a 1-layer LTSM for FMOW. Then a two-layer MLP is used as the classifier. 

\textbf{Multiple Source Domain Adaptation (MDA)}~~
We compare with algorithms designed for MDA, where the algorithm has simultaneous access to multiple labeled domains $t = 1, \ldots, T - 1$ and learns to adapt to the target domain. We use \textit{maxmin} and \textit{dynamic} variants of MDAN \citep{zhao2018adversarial}, Fish \citep{shi2022gradient} and DARN \citep{wen2020domain} as our baseline algorithms. We also provide the comparisons where the MDA algorithm only has access to the last two source domains ($T - 2, T-1$).

\textbf{Gradual Adaptations}~~
Our proposed approach trains the algorithm to sequentially adapt from the initial domain $t = 1$ to the last domain $t=T$. At each time step, the algorithm only has access to two consecutive domains. 
We use cross-entropy loss with objective (\ref{eq:onesteploss}) to perform successive adaptations with gradually changing domains. The rest of the setup (e.g. model for feature extraction and classifier) is the same as the ones for Direct Adaptation. 

\textbf{Meta Learning}~~
We also compare our algorithm with algorithms designed for Meta-learning with continuously evolving data \citep{liu2020learning}. Specifically, the method continually performs adaptations and aims to learn a meta-representation for evolving target domains. However, different from our objective, the method also tries to preserve performance from previous target domains as it adapts. We compare to EAML proposed by \citet{liu2020learning} and evaluated on the rotating MNIST dataset with 5 domains. For implementation, we used LeNet for learning the meta-representation and a two-layer classifier as the meta-adapter (as instructed in the original paper). For a fair comparison, we let the batch size be 64. All other parameters are kept the same according to their official code release. 

\section{More experimental results}
In this section, we present additional experimental results and compare our algorithm with our baseline algorithms with shifting MNIST.  
\paragraph{Shifting MNIST}
The shifting MNIST is a semi-synthetic dataset from the MNIST dataset. 
For the shifting MNIST, we shift the pixel to up, right, down, and left by $1, 3, 5$ pixels. Each domain contains images that have been shifted in at least one direction and each consecutive domains contain images shifted in different directions.
\begin{table*}[ht]\centering
\caption{Results on shifting MNIST dataset with Gradual Adaptation on 5 domains, Direct Adaptation, and No Adaptation. }\label{table:shift}
\begin{tabular}{@{}cccccc@{}}
\toprule
\multicolumn{6}{c}{Shifting MNIST}                                                                                                                                                                                \\ \midrule
\multicolumn{1}{c|}{}         & \multicolumn{2}{c|}{Gradual Adaptation with 5 domains}                         & \multicolumn{2}{c|}{Direct Adpatation}                                         & No Adaptation    \\ \midrule
\multicolumn{1}{c|}{}         & \multicolumn{1}{c|}{CNN}              & \multicolumn{1}{c|}{LSTM}             & \multicolumn{1}{c|}{CNN}              & \multicolumn{1}{c|}{LSTM}              & CNN              \\ \midrule
\multicolumn{1}{c|}{1 pixels} & \multicolumn{1}{c|}{86.19 $\pm$ 1.54} & \multicolumn{1}{c|}{\textbf{93.81} $\pm$ 0.35} & \multicolumn{1}{c|}{76.72 $\pm$ 0.70} & \multicolumn{1}{c|}{89.13 $\pm$ 1.10}  & 87.36 $\pm$ 1.43 \\ \midrule
\multicolumn{1}{c|}{3 pixels} & \multicolumn{1}{c|}{73.61 $\pm$ 2.53} & \multicolumn{1}{c|}{\textbf{79.11} $\pm$ 0.95}  & \multicolumn{1}{c|}{60.12 $\pm$ 1.07} & \multicolumn{1}{c|}{74.42 $\pm $ 1.89} & 70.63 $\pm$ 2.56 \\ \midrule
\multicolumn{1}{c|}{5 pixels} & \multicolumn{1}{c|}{62.80 $\pm$ 1.81}             & \multicolumn{1}{c|}{\textbf{71.31} $\pm$ 2.55} & \multicolumn{1}{c|}{46.83 $\pm$ 0.56} & \multicolumn{1}{c|}{64.25 $\pm$ 2.57}  & 44.41 $\pm$ 2.68 \\ \bottomrule
\end{tabular}
\end{table*}

\begin{table*}[ht]\centering
\caption{Results on shifting MNIST dataset with Gradual Adaptation on 5 domains and Multiple Source Domain Adaptation method (MDA) \citep{zhao2018adversarial}  }\label{table:shift2}
\begin{tabular}{@{}cccccc@{}}
\toprule
\multicolumn{6}{c}{Shifting MNIST}                                                                                                                                                                                                                                                      \\ \midrule
\multicolumn{1}{c|}{}         & \multicolumn{2}{c|}{\begin{tabular}[c]{@{}c@{}}Gradual Adaptation \\ with 5 domains\end{tabular}} & \multicolumn{3}{c}{MDA}                                                                                                                              \\ \cmidrule(l){2-6} 
\multicolumn{1}{c|}{}         & \multicolumn{1}{c|}{CNN}                        & \multicolumn{1}{c|}{LSTM}                      & \multicolumn{1}{c|}{Maxmin}           & \multicolumn{1}{c|}{Dynamic}          & \begin{tabular}[c]{@{}c@{}}Dynamic\\ with last 2 domain\end{tabular} \\ \midrule
\multicolumn{1}{c|}{1 pixel}  & \multicolumn{1}{c|}{86.19 $\pm$ 1.54}           & \multicolumn{1}{c|}{93.81 $\pm$ 0.35}          & \multicolumn{1}{c|}{96.44 $\pm$ 0.72} & \multicolumn{1}{c|}{\textbf{97.85} $\pm$ 0.25} & 95.43 $\pm$ 0.82                                                     \\ \midrule
\multicolumn{1}{c|}{3 pixels} & \multicolumn{1}{c|}{73.61 $\pm$ 2.53}           & \multicolumn{1}{c|}{79.11 $\pm$ 0.95}          & \multicolumn{1}{c|}{\textbf{93.79} $\pm$ 0.34} & \multicolumn{1}{c|}{93.14 $\pm$ 4.45} & 78.21 $\pm$ 0.69                                                     \\ \midrule
\multicolumn{1}{c|}{5 pixels} & \multicolumn{1}{c|}{62.80 $\pm$ 1.81}           & \multicolumn{1}{c|}{71.31 $\pm$ 2.55}          & \multicolumn{1}{c|}{90.49 $\pm$ 0.72} & \multicolumn{1}{c|}{\textbf{93.49} $\pm$ 0.55} & 59.92 $\pm$ 0.92                                                     \\ \bottomrule
\end{tabular}
\end{table*}
We observe that when the shift in the domain is subtle ($1$ pixel), directly performing empirical risk minimization (No Adaptation) can even outperform the Gradual Adaptation with CNN model architecture. This showcases the insufficiency of the one-step loss. However, we note that with a temporal summarization, the one-step loss still overtakes the No Adaptation and Direct Adaptation methods. This is reflected through Gradual Adaptation with LSTM performs the best when compared to the No Adaptation and Direct Adaptation methods in all cases. 

On the shifting MNIST, results show that simultaneous access to all data domains is crucial, as MDA methods are more successful in all cases. This is in contrast with the Rotating MNIST dataset, for which we showed in an earlier section that gradual domain adaptation can be better than multiple source domain adaptation while relaxing the requirement of simultaneous access to the training domains. However, we note that Gradual Adaptation still performs better than the MDA method when it is restricted to accessing the last two domains on $3$ and $5$ pixels shift. This shows the potential of gradual domain adaptation when only sequential access to data domains is available. 

\paragraph{Experiment details}
We conducted the experiments on 2 GeForce RTX 3090, each with 24265 MB memory and with a CUDA Version of 11.4. Each set of experiments is repeated with 5 different random seeds to ensure reproducible results. We report experimental details (i.e. hyperparameters) here.

We penalize the gradient using the gradient penalty method originally designed for stable generative model training \citep{gulrajani2017improved}. We denote the coefficient of penalty as the GP factor. To stabilize the primal-dual training, we perform gradient descent on the critic side for multiple steps before descending on the classifier. We denote the number of times we descent the critic side per descent on the classifier side as $k$-critic. Algorithms are trained til their loss converges.

For the MNIST (rotating and shifting) dataset and Portrait dataset, we've used the following configurations
\begin{itemize}
    \item Learning rates: $1e-3$ (classifier and feature extractor), $5e-4$ (critic).
    \item Batch size: 64
    \item GP factor: 5
    \item $k$-critic: 5
    \item Temporal model: 2 layer GRU with a hidden size of 64.
\end{itemize}

For the FMOW dataset, the following configurations are used. We further normalize the data before feed into the LSTM layer for better performance.
\begin{itemize}
    \item Learning rates: $1e-3$ (classifier and feature extractor), $1e-5$ (critic).
    \item Gradient clipping: 3
    \item Batch size: 216
    \item GP factor: 1
    \item $k$-critic: 5
    \item Temporal model: 1 layer LSTM with a hidden size of 2048.
\end{itemize}

\end{document}